\documentclass[twoside,3p]{elsarticle}

\usepackage[T1]{fontenc}
\usepackage[utf8]{inputenc}
\usepackage{geometry}
\usepackage{amssymb}
\usepackage{amsmath}
\usepackage{amsthm}
\usepackage{bm}

\usepackage{graphicx}
\usepackage{booktabs}
\usepackage{multirow}
\usepackage{makecell}
\usepackage{multicol}
\usepackage{adjustbox}
\usepackage{algorithm}
\usepackage{algpseudocode}
\usepackage{xcolor}

\usepackage{tikz}
\usetikzlibrary{arrows.meta,positioning,calc,shapes.geometric,patterns,decorations.pathmorphing,backgrounds,fit,3d}
\usepackage{pgfplots}
\pgfplotsset{compat=1.16}

\usepackage{hyperref}
\hypersetup{colorlinks=true, allcolors=blue}
\usepackage[nameinlink]{cleveref}
\crefname{figure}{Fig.}{Figs.}
\crefformat{equation}{Eq.~#2(#1)#3}
\crefformat{section}{Section~#2#1#3}
\AtBeginDocument{\let\citet\cite}

\usepackage[labelfont=bf]{caption}
\usepackage{subcaption}

\newtheorem{prop}{Proposition}

\tikzset{
  neuron/.style={circle,draw=black,minimum size=6mm,inner sep=0pt,fill=white},
  block/.style={rounded corners,draw=black,thick,align=center,inner sep=5pt},
  edgearrow/.style={-{Latex[length=2mm]},thick},
}

\begin{document}

\begin{frontmatter}

\title{PG-KINN: A Physics-Informed Petrov--Galerkin Kolmogorov--Arnold Network for Solving Forward and Inverse PDEs}

\author[inst1,inst2]{Amirhossein Sadr}
\author[inst1]{Nima Soltani}
\author[inst3]{Vahideh Moghtadaiee}
\author[inst2]{Aida Pakniyat}
\author[inst1]{Dara Rahmati\corref{cor1}}
\ead{d_rahmati@sbu.ac.ir}
\author[inst4]{Saeid Gorgin}

\cortext[cor1]{Corresponding author}

\address[inst1]{Department of Computer Science and Engineering, Shahid Beheshti University, Tehran, Iran}
\address[inst2]{School of Computer Science, Institute for Research in Fundamental Sciences (IPM), Tehran, Iran}
\address[inst3]{Cyberspace Research Institute, Shahid Beheshti University, Tehran, Iran}
\address[inst4]{School of Physics, Engineering and Computer Science, University of Hertfordshire, Hatfield, UK}

\begin{abstract}
Physics-informed learning of partial differential equations (PDEs) has been dominated by multilayer perceptrons (MLPs), whose spectral bias and dense parameterization limit both accuracy and interpretability. Kolmogorov--Arnold Networks (KANs) mitigate these limitations because their learnable spline activations are structurally aligned with the piecewise-polynomial bases of classical discretizations. However, the way a PDE is cast into a loss functional is as decisive as the choice of approximator: strong-form residual minimization requires high-order derivatives and heavily weighted losses, the energy (Bubnov--Galerkin) form is restricted to self-adjoint operators and, as we show, collapses to a trivial solution for parameter-identification problems, and boundary-integral forms require a known fundamental solution. We propose \emph{PG-KINN}, a physics-informed KAN built on a \emph{Petrov--Galerkin} formulation in which the trial space is a KAN and the test space is an \emph{independent}, compactly supported, piecewise-polynomial space evaluated with Gauss--Legendre quadrature. Integration by parts lowers the differentiation order while retaining applicability to general non-self-adjoint, nonlinear, and inverse problems; the localized test functions turn the global residual into a set of element-wise weak residuals with favorable conditioning. We prove that the resulting residual-based inverse loss is immune to the zero-modulus collapse that afflicts energy-based inverse formulations. On a suite of benchmarks spanning crack singularities, stress concentration, Neo-Hookean hyperelasticity, inverse parameter identification in heterogeneous media, and complex geometries---PG-KINN consistently outperforms legacy MLP baselines and state-of-the-art KAN-based strong/energy/inverse formulations (PIKAN). The one exception is extreme complex geometries, where the axis-aligned spline grid remains a bottleneck. These results position the Petrov--Galerkin coupling of KAN trial spaces and polynomial test spaces as a robust and accurate route for AI-based computational mechanics.
\end{abstract}

\begin{keyword}
Physics-informed neural networks \sep Kolmogorov--Arnold Networks \sep Petrov--Galerkin method \sep Weak-form PDE solvers \sep Inverse problems \sep Computational mechanics
\end{keyword}

\end{frontmatter}

\section*{Nomenclature}
\begin{multicols}{2}
	\begin{description}
		\item[AD] Automatic differentiation
		\item[BINN] Boundary-integral neural network (inverse form)
		\item[DEM] Deep energy method (Bubnov--Galerkin/energy form)
		\item[FEM] Finite element method
		\item[PG-KINN] KAN trial space in a Petrov--Galerkin weak form (this work)
		\item[GL] Gauss--Legendre quadrature
		\item[KAN] Kolmogorov--Arnold Network
		\item[MLP] Multilayer perceptron
		\item[NTK] Neural tangent kernel
		\item[PDE] Partial differential equation
		\item[PG] Petrov--Galerkin
		\item[PIKAN] Physics-informed KAN (strong/energy/inverse form with KAN backbone)
		\item[PINN] Physics-informed neural network (strong form)
		\item[VPINN] Variational PINN (weak form with MLP)
	\end{description}
\end{multicols}

\section{Introduction}\label{sec:intro}
Partial differential equations (PDEs) are the language of continuum physics, and their reliable solution underpins engineering analysis and design. When geometries, material laws, or boundary conditions become intricate, closed-form solutions are unavailable and numerical methods such as the finite element method (FEM) \cite{finite_element_book,hughes2012finite,bathe2006finite,reddy2019introduction}, mesh-free methods \cite{liu2003mesh,rabczuk2004cracking,nguyen2008meshless}, finite-difference \cite{leveque2007finitedifferentialmethod}, finite-volume \cite{darwish2016finitevolumemethod}, and boundary-element methods \cite{brebbia2012boundary} become indispensable. Over the last decade, deep learning has emerged as a complementary paradigm for PDEs, broadly organized into physics-informed neural networks (PINNs) \cite{PINN_original_paper}, operator learning \cite{DeepOnet,li2020fourier}, and physics-informed operators \cite{li2024physics}. The present work concerns the first paradigm, in which a single network is trained to satisfy the governing equations of one boundary-value problem in a mesh-free manner.

A recurring theme in this literature is that the same PDE admits multiple, mathematically equivalent but computationally distinct, variational statements. The strong form underlies collocation-based PINNs \cite{PINN_original_paper}; the weak form underlies variational PINNs (hp-VPINNs) \cite{hp-VPINN,kharazmi2019variational,khodayi2020varnet,berrone2022variational}; the energy form underlies the deep energy method (DEM) \cite{loss_is_minimum_potential_energy,deep_ritz,wang2022cenn}; and the inverse (boundary-integral) form underlies boundary-integral neural networks (BINNs) \cite{sun2023binn}. Each statement trades derivative order, boundary-condition treatment, and the class of admissible operators against one another, so that the choice of formulation is not a matter of taste but of numerical performance \cite{the_comparision_of_strong_and_energy_form}.

The second ingredient of any physics-informed solver is the function approximator. The universal approximation theorem justifies the multilayer perceptron (MLP) \cite{hornik1989multilayer,super_approximation}, but MLPs suffer from a well-documented spectral bias: through the lens of the neural tangent kernel (NTK), high-frequency components are learned far more slowly than low-frequency ones \cite{jacot2018neural,NTK_PINN,rahaman2019spectral}. The recently proposed Kolmogorov--Arnold Network (KAN) \cite{liu2024kan}, grounded in the Kolmogorov--Arnold representation theorem \cite{kolmogorov1961representation,braun2009constructive}, replaces fixed nodal nonlinearities with learnable univariate functions, typically B-splines. This yields fewer parameters, greater interpretability, and---crucially for PDEs---a trial space whose local, piecewise-polynomial structure mirrors the shape functions of FEM and the NURBS bases of isogeometric analysis \cite{hughes2005isogeometric}. Variants replace B-splines with Chebyshev polynomials \cite{ss2024chebyshev,shukla2024comprehensive}, radial basis functions \cite{li2024kolmogorov}, or wavelets \cite{bozorgasl2024wav}, and early studies confirm KAN's promise for differential equations \cite{rojas2024adaptive,wang2024expressiveness}.

Coupling KANs with the different PDE formulations has recently been shown to improve accuracy and convergence over MLP counterparts on strong-, energy-, and inverse-form problems \cite{liu2024kan}. Yet each of these three formulations carries an intrinsic limitation that KAN alone does not remove: the strong form still requires the highest derivative order and a delicate balancing of residual and boundary penalties; the energy form still requires the operator to be self-adjoint and, as we prove in \Cref{subsec:collapse}, becomes ill-posed for material-identification inverse problems; and the boundary-integral form still requires a closed-form fundamental solution, which is unavailable for most nonlinear or heterogeneous operators. A formulation that simultaneously (i) reduces the derivative order, (ii) remains valid for general non-self-adjoint and nonlinear operators, (iii) avoids fundamental solutions, and (iv) is robust for inverse problems, would combine the strengths of all three while inheriting the approximation advantages of KAN.

We argue that the classical \emph{Petrov--Galerkin} method provides exactly this formulation. In a Petrov--Galerkin statement the trial and test spaces are chosen \emph{independently}: we let the trial space be a KAN and the test space be a fixed, compactly supported, piecewise-polynomial space defined on a background partition and integrated with Gauss--Legendre quadrature. Integration by parts transfers derivatives from the KAN trial function onto the smooth polynomial test functions, lowering the differentiation order (as in the weak/energy forms) while preserving generality (as in the strong form). Because the test functions are local, the single global residual is replaced by a vector of element-wise weak residuals, which improves conditioning and provides a natural, penalty-light treatment of interfaces and Neumann data.

We term the resulting framework \textbf{PG-KINN}. Its contributions are:
\begin{enumerate}
	\item A Petrov--Galerkin physics-informed formulation with a KAN trial space and an independent piecewise-polynomial test space, unifying the derivative-order reduction of weak forms with the operator generality of strong forms.
	\item An inverse formulation whose residual-based loss we prove to be free of the zero-modulus collapse intrinsic to energy-based inverse problems.
	\item A systematic evaluation on singularity, stress-concentration, nonlinear hyperelastic, heterogeneous inverse, and complex-geometry benchmarks, showing consistent accuracy gains over both legacy MLP baselines and state-of-the-art KAN-based strong/energy/inverse formulations, and an honest account of the extreme-complex-geometry limitation.
\end{enumerate}

The remainder of the paper is organized as follows. \Cref{sec:motivation} motivates the Petrov--Galerkin choice. \Cref{sec:prelim} reviews KANs, the strong form, the Petrov--Galerkin formulation, and its inverse counterpart. \Cref{sec:framework} introduces PG-KINN, including the weak-form derivation, the collapse-freeness proof, the KAN--FEM correspondence, and quadrature/test-function details. \Cref{sec:experiments} reports numerical experiments, and \Cref{sec:conclusion} concludes.

\section{Motivation}\label{sec:motivation}
The design of a physics-informed solver rests on two coupled decisions---how to \emph{represent} the solution and how to \emph{measure} its physical consistency. We motivate PG-KINN by examining the failure modes that arise when either decision is made in isolation, and by showing that a Petrov--Galerkin coupling of a KAN trial space with an independent polynomial test space resolves them jointly.

\paragraph{Representation: why splines, not MLPs}
An MLP trained by gradient descent converges at a rate governed by the eigenspectrum of its NTK, whose eigenvalues decay rapidly so that high-frequency error modes persist \cite{jacot2018neural,NTK_PINN}. Because differentiation amplifies high frequencies, this spectral bias is especially damaging when the loss involves derivatives of the network, as in every physics-informed formulation. KANs replace global nodal nonlinearities with local B-spline activations; the compact support of the splines flattens the NTK spectrum and equips the trial space with the same piecewise-polynomial locality that makes FEM and isogeometric analysis effective for fields with sharp gradients, interfaces, and singularities. This makes a spline-based KAN the natural trial space for mechanics-oriented PDEs.

\paragraph{Measurement: the limits of a single formulation}
Given a good trial space, three established formulations still each fall short.
\emph{(i) Strong form.} Collocation minimizes the pointwise residual of the highest-order operator, which magnifies approximation error near singularities and forces a fragile, hand-tuned balance between interior and boundary penalties \cite{ill_gradient,NTK_to_get_hyperparameter_of_PINN}.
\emph{(ii) Energy form.} Minimizing a potential energy lowers the derivative order and eliminates loss weights, but it presupposes a self-adjoint operator that admits an energy functional \cite{loss_is_minimum_potential_energy}; for parameter identification it is worse than fragile---we show in \Cref{subsec:collapse} that minimizing strain energy over an unknown modulus drives the modulus to zero, a spurious global minimizer.
\emph{(iii) Inverse/boundary-integral form.} Recasting the problem on the boundary is highly accurate but requires the fundamental solution of the operator \cite{sun2023binn}, which is unavailable for nonlinear or strongly heterogeneous media.

\paragraph{The Petrov--Galerkin resolution}
The weighted-residual (weak) statement multiplies the residual by a test function and integrates by parts, transferring one or more derivatives onto the test function. Choosing the test space \emph{independently} of the trial space---the defining feature of a Petrov--Galerkin method, in contrast to the Bubnov--Galerkin/energy choice of test $=$ variation of trial---yields four decisive advantages for a KAN-based solver:
\begin{itemize}
	\item \textbf{Lower derivative order without loss of generality.} Integration by parts reduces the order demanded of the KAN (as in the energy form) while keeping the formulation valid for non-self-adjoint and nonlinear operators (as in the strong form), and without any fundamental solution (unlike BINN).
	\item \textbf{Localized, well-conditioned residuals.} Compactly supported polynomial test functions produce one weak residual per background cell. The resulting vector-valued loss is better conditioned than a single global residual and localizes error control where the physics is difficult, echoing the element-wise assembly of FEM.
	\item \textbf{Penalty-light boundary and interface handling.} Neumann data enter naturally as boundary integrals, and material interfaces are captured by the intrinsic $C^0$ structure of the spline trial space, reducing the number of hand-tuned weights relative to strong-form losses.
	\item \textbf{Robust inverse problems.} When the unknown is a material field, the weak residual is affine in that field, giving a genuine least-squares problem whose minimizer is the true field; the trivial (zero) solution is \emph{not} a minimizer.
\end{itemize}
The synergy is specific: KAN supplies a spline trial space aligned with the geometry of the solution, and the Petrov--Galerkin weak form supplies a measurement of physical consistency that is low-order, general, and robust. \Cref{fig:motivation} summarizes this design space. PG-KINN is the point at which both decisions are made together.

\tikzset{
  neuron/.style={circle,draw=black,minimum size=6mm,inner sep=0pt,fill=white},
  block/.style={rounded corners,draw=black,thick,align=center,inner sep=5pt,fill=gray!15},
  edgearrow/.style={-{Latex[length=2mm]},thick,black},
}

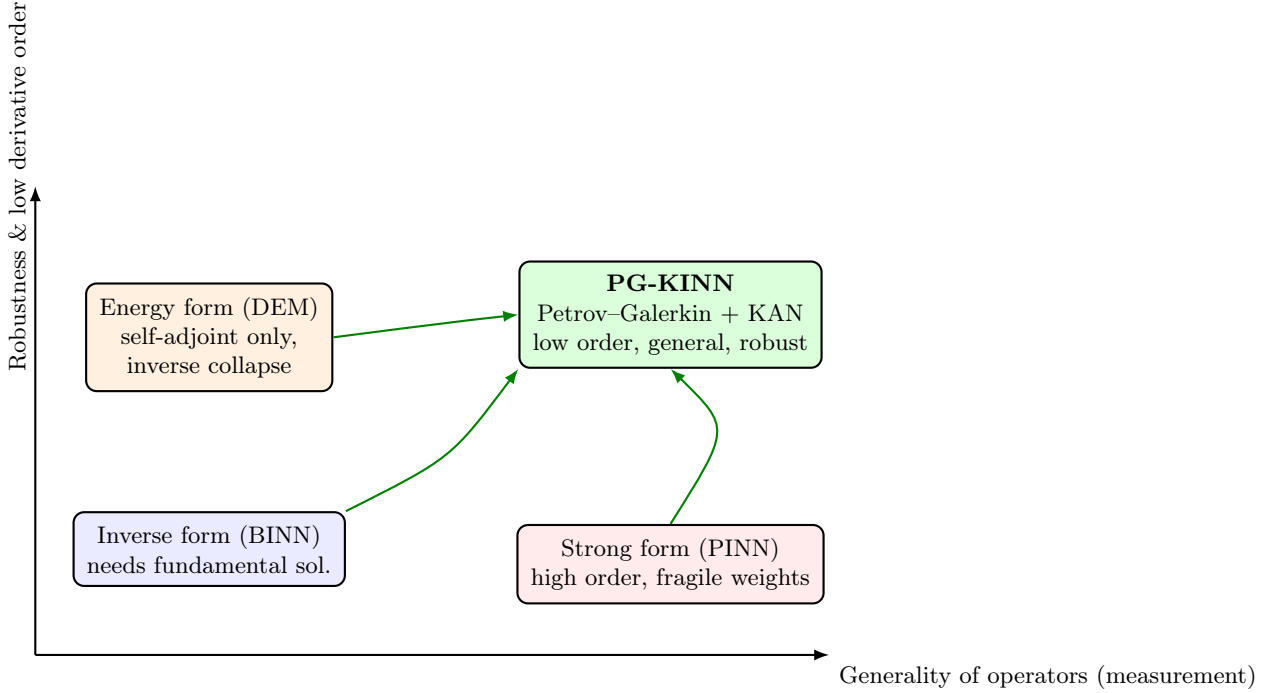
\begin{figure}[t]
\centering
\begin{tikzpicture}[font=\small]
  \draw[edgearrow] (0,0) -- (10.5,0) node[below right]{Generality of operators (measurement)};
  \draw[edgearrow] (0,0) -- (0,6.2) node[above left,rotate=90,anchor=south]{Robustness \& low derivative order};
  \node[block,fill=red!8] (strong) at (8.4,1.2) {Strong form (PINN)\\ high order, fragile weights};
  \node[block,fill=orange!12] (energy) at (2.3,4.2) {Energy form (DEM)\\ self-adjoint only,\\ inverse collapse};
  \node[block,fill=blue!8] (binn) at (2.3,1.4) {Inverse form (BINN)\\ needs fundamental sol.};
  \node[block,fill=green!14,thick] (gk) at (8.4,4.5) {\textbf{PG-KINN}\\ Petrov--Galerkin $+$ KAN\\ low order, general, robust};
  \draw[edgearrow,green!50!black] (strong.north) .. controls (9.2,3) .. (gk.south);
  \draw[edgearrow,green!50!black] (energy.east) .. controls (5.5,4.4) .. (gk.west);
  \draw[edgearrow,green!50!black] (binn.north east) .. controls (5.5,2.6) .. (gk.south west);
\end{tikzpicture}
\caption{The design space of physics-informed formulations. Strong, energy, and inverse forms each occupy a corner with a structural limitation. PG-KINN combines a KAN trial space with a Petrov--Galerkin test space to reach the favorable region: low derivative order, generality across operators, and robustness for inverse problems.}
\label{fig:motivation}
\end{figure}

\section{Preliminaries}\label{sec:prelim}
This section fixes notation and reviews the ingredients of PG-KINN: the KAN trial space, the strong form of a boundary-value problem, the Petrov--Galerkin weak form, and its inverse counterpart.

\subsection{Kolmogorov--Arnold Networks}\label{subsec:kan}
The Kolmogorov--Arnold representation theorem states that any continuous multivariate function can be written as a finite superposition of continuous univariate functions \cite{kolmogorov1961representation}. KAN operationalizes this by making the univariate functions learnable \cite{liu2024kan}. For a layer with $l_i$ inputs and $l_o$ outputs, each scalar map $\phi_{ij}$ is expanded in a B-spline basis of order $r$ on a grid of size $G$,
\begin{equation}
	\phi_{ij}(x_j)=\sum_{m=1}^{G+r} c^{(i,j)}_{m}\,B_{m}(x_j),
	\label{eq:spline_activation}
\end{equation}
where $\{B_m\}$ are the B-spline basis functions and $c^{(i,j)}_m$ are trainable coefficients. Following \cite{liu2024kan}, a residual path and scaling factors are added so that the layer output reads
\begin{equation}
	\bm{Y}=\Big[\textstyle\sum_{\text{col}}\bm{\phi}(\bm{X})\odot \bm{S}\Big]+\bm{W}\,\sigma(\bm{X}),
	\label{eq:kan_layer}
\end{equation}
with $\odot$ the element-wise product, $\bm{S}$ and $\bm{W}$ scaling/linear factors, and $\sigma$ a smooth activation (SiLU). Stacking $N$ such layers gives the KAN map $u_h(\bm{x};\bm{\theta})=K^{(N)}\!\circ\cdots\circ K^{(1)}(\bm{x})$. To keep intermediate outputs inside the spline grid $[-1,1]$, a $\tanh$ is applied after every hidden layer, and the input is affinely rescaled by the bounding box $[L,W,X_c,Y_c]$ of the domain,
\begin{equation}
	x^{s}=\frac{x-X_c}{L/2},\qquad y^{s}=\frac{y-Y_c}{W/2}.
	\label{eq:normalize}
\end{equation}
The trainable parameters are the spline coefficients $c^{(i,j)}_m$ (count $l_i l_o (G+r)$), the base weights $\bm{W}$, and the scalers $\bm{S}$.

\subsection{Strong form of PDEs}\label{subsec:strong}
Consider a boundary-value problem on a domain $\Omega\subset\mathbb{R}^d$ with boundary $\Gamma=\Gamma^u\cup\Gamma^t$,
\begin{equation}
	\begin{cases}
		\mathcal{P}(\bm{u}(\bm{x}))=\bm{f}(\bm{x}), & \bm{x}\in\Omega,\\
		\bm{u}(\bm{x})=\bar{\bm{u}}(\bm{x}), & \bm{x}\in\Gamma^{u},\\
		\mathcal{B}(\bm{u}(\bm{x}))=\bar{\bm{t}}(\bm{x}), & \bm{x}\in\Gamma^{t},
	\end{cases}
	\label{eq:strong}
\end{equation}
where $\mathcal{P}$ is the (possibly nonlinear) interior operator, $\mathcal{B}$ the Neumann/flux operator, and $\Gamma^u,\Gamma^t$ the Dirichlet and Neumann boundaries. Strong-form PINNs choose the residual itself as the weight function and minimize
\begin{equation}
	\mathcal{L}_{\mathrm{strong}}=\frac{\lambda_r}{N_r}\sum_{i=1}^{N_r}\big|\mathcal{P}(u_h(\bm{x}_i))-\bm{f}(\bm{x}_i)\big|^2
	+\frac{\lambda_b}{N_b}\sum_{i=1}^{N_b}\big|\mathcal{B}(u_h(\bm{x}_i))-\bar{\bm{t}}(\bm{x}_i)\big|^2 + \cdots,
	\label{eq:strong_loss}
\end{equation}
which requires the full order of $\mathcal{P}$ and a set of weights $\{\lambda\}$ that must be tuned. This is our most general but most derivative-hungry baseline.

\subsection{The Petrov--Galerkin formulation}\label{subsec:pg}
Multiplying the interior residual by a test function $v$ from a test space $\mathcal{V}$ and integrating over $\Omega$ gives the weighted-residual statement
\begin{equation}
	\int_\Omega \big[\mathcal{P}(\bm{u})-\bm{f}\big]\,v\,d\Omega = 0,\qquad \forall v\in\mathcal{V}.
	\label{eq:weighted_residual}
\end{equation}
For a second-order operator such as $\mathcal{P}(u)=-\nabla\!\cdot(\kappa\nabla u)$, integration by parts (Green's identity) transfers one derivative to $v$,
\begin{equation}
	\int_\Omega \kappa\,\nabla u\cdot\nabla v\,d\Omega
	-\int_{\Gamma^{t}}\bar{t}\,v\,d\Gamma
	-\int_\Omega f\,v\,d\Omega = 0,\qquad \forall v\in\mathcal{V},
	\label{eq:weakform}
\end{equation}
where Neumann data enter as a natural boundary integral and $v$ is required to vanish on $\Gamma^u$; the full derivation for a general operator is given in \Cref{subsec:weakform}. The essential distinction is the choice of the trial space $\mathcal{U}$ (which contains $u$) and the test space $\mathcal{V}$:
\begin{itemize}
	\item In the \emph{Bubnov--Galerkin} (energy) method, $\mathcal{V}$ is taken as the variation $\delta u$ of the trial space, so trial and test spaces coincide. This yields the deep energy method and is restricted to operators possessing an energy functional.
	\item In the \emph{Petrov--Galerkin} method adopted here, $\mathcal{V}$ is chosen \emph{independently} of $\mathcal{U}$. We set $\mathcal{U}$ to be the KAN of \Cref{subsec:kan} and $\mathcal{V}=\mathrm{span}\{v_k\}_{k=1}^{M}$ to be a fixed, compactly supported, piecewise-polynomial space on a background partition $\{\Omega_e\}$ of $\Omega$.
\end{itemize}
Enforcing \Cref{eq:weakform} against each basis test function $v_k$ produces a vector of weak residuals
\begin{equation}
	R_k(\bm{\theta})=\int_\Omega \kappa\,\nabla u_h(\bm{x};\bm{\theta})\cdot\nabla v_k\,d\Omega
	-\int_{\Gamma^{t}}\bar{t}\,v_k\,d\Gamma-\int_\Omega f\,v_k\,d\Omega,\qquad k=1,\dots,M,
	\label{eq:weak_residuals}
\end{equation}
and the network is trained by minimizing their squared sum,
\begin{equation}
	\mathcal{L}_{\mathrm{PG}}(\bm{\theta})=\frac{1}{M}\sum_{k=1}^{M} \omega_k\,R_k(\bm{\theta})^2
	+ \lambda_u\,\mathcal{L}_{\Gamma^u}(\bm{\theta}),
	\label{eq:pg_loss}
\end{equation}
where $\mathcal{L}_{\Gamma^u}$ enforces the Dirichlet data (by penalty or by an admissible construction) and $\omega_k$ are optional per-cell weights (unity unless stated). All integrals in \Cref{eq:weak_residuals} are evaluated by Gauss--Legendre quadrature on the background cells (\Cref{subsec:quadrature}). Because each $v_k$ is local, $R_k$ probes the physics on a small patch, and the aggregate loss behaves like an element-wise assembled residual rather than a single global scalar. \Cref{fig:pg} contrasts the Bubnov-- and Petrov--Galerkin choices.

\begin{figure}[t]
\centering
\begin{tikzpicture}[font=\small]
  \node[block,fill=green!8,minimum width=3.6cm,minimum height=2.3cm] (trial) at (0,0)
    {\textbf{Trial space }$\mathcal{U}$\\[2pt] KAN (B-splines)\\ globally smooth\\ trainable $\bm{\theta}$};
  \node[block,fill=blue!8,minimum width=3.6cm,minimum height=2.3cm] (test) at (6.2,0)
    {\textbf{Test space }$\mathcal{V}$\\[2pt] piecewise-polynomials\\ compact support\\ fixed (no training)};
  \draw[edgearrow] (trial) -- node[above,align=center]{weak form\\ $\int(\cdots)v_k$} (test);
  \node[align=center] at (0,-2.4) {\textbf{Bubnov--Galerkin}\\ $\mathcal{V}=\delta\mathcal{U}$ (energy)\\ self-adjoint only};
  \node[align=center] at (6.2,-2.4) {\textbf{Petrov--Galerkin (ours)}\\ $\mathcal{V}\neq\mathcal{U}$ (independent)\\ general operators};
  \draw[thick,dashed] (3.1,-1.4) -- (3.1,-3.1);
\end{tikzpicture}
\caption{Petrov--Galerkin versus Bubnov--Galerkin. PG-KINN keeps the trial space (KAN) and the test space (compactly supported piecewise-polynomials) independent, so the weak form applies to general---including non-self-adjoint and nonlinear---operators, unlike the energy form in which the test space is the variation of the trial space.}
\label{fig:pg}
\end{figure}
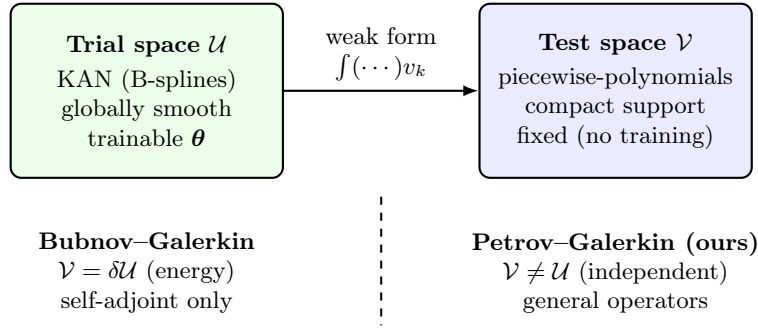

\subsection{Inverse formulation in Petrov--Galerkin}\label{subsec:inverse_pg}
In an inverse problem the field $u$ (and possibly $f$, $\bar t$) is measured while a material coefficient---say the conductivity $\kappa(\bm{x})$ in $-\nabla\!\cdot(\kappa\nabla u)=f$---is unknown. We represent the coefficient by a KAN, $\kappa_h(\bm{x};\bm{\theta}_\kappa)$, and substitute it into the weak residual \Cref{eq:weak_residuals} using the \emph{measured} field for $u$:
\begin{equation}
	R_k(\bm{\theta}_\kappa)=\int_\Omega \kappa_h(\bm{x};\bm{\theta}_\kappa)\,\nabla u\cdot\nabla v_k\,d\Omega
	-\int_{\Gamma^{t}}\bar{t}\,v_k\,d\Gamma-\int_\Omega f\,v_k\,d\Omega,
	\label{eq:inverse_residual}
\end{equation}
and minimize $\mathcal{L}_{\mathrm{PG}}^{\mathrm{inv}}(\bm{\theta}_\kappa)=\frac{1}{M}\sum_{k}R_k(\bm{\theta}_\kappa)^2$. Two features are important. First, since $u$ is fixed data, $R_k$ is \emph{affine} in $\kappa_h$: the map $\bm{\theta}_\kappa\mapsto R_k$ inherits the linearity of the operator in $\kappa$, so the loss is a genuine linear least-squares residual (up to the network parameterization) whose global minimizer reproduces the true coefficient. Second, and in stark contrast to the energy form, the trivial field $\kappa_h\equiv 0$ is \emph{not} a minimizer: setting $\kappa_h=0$ leaves the nonzero data terms $\int_{\Gamma^t}\bar t\,v_k+\int_\Omega f\,v_k$, so $R_k\neq 0$. We formalize this collapse-freeness in \Cref{subsec:collapse}, contrasting it with the energy-based inverse problem in which the strain-energy minimizer drives the modulus to zero. This is the central reason PG-KINN is well suited to parameter identification in heterogeneous media.

\section{The Proposed Framework: PG-KINN}\label{sec:framework}
PG-KINN assembles the ingredients of \Cref{sec:prelim} into a single, general-purpose solver for forward and inverse PDEs. The trial field is a KAN, the physics is measured by a Petrov--Galerkin weak residual against an independent piecewise-polynomial test space, and all integrals are evaluated by Gauss--Legendre quadrature. \Cref{fig:framework} shows the pipeline and \Cref{alg:PG-KINN_forward,alg:PG-KINN_inverse} state the procedures.

\paragraph{Trial space}
The solution (forward) or the unknown coefficient (inverse) is approximated by a KAN $u_h(\bm{x};\bm{\theta})$ as in \Cref{eq:kan_layer}, with inputs normalized to the spline grid by \Cref{eq:normalize}. The spline grid of the KAN and the background quadrature partition are chosen independently; the former controls the resolution of the trial field, the latter the accuracy of the weak-form integrals.

\paragraph{Test space}
On a background partition $\{\Omega_e\}_{e=1}^{E}$ of $\Omega$ we define $M$ compactly supported test functions $\{v_k\}$. Two families are used: (i) nodal hat/tensor-product functions of degree $p$ that are $C^0$ across cell faces, and (ii) higher-order interior ``bubble'' functions built from shifted Legendre polynomials $\{L_p\}$ that vanish on $\partial\Omega_e$,
\begin{equation}
	v^{\mathrm{bub}}_{e,p}(\bm{\xi}) = (1-\xi_1^2)(1-\xi_2^2)\,L_{p}(\xi_1)L_{p}(\xi_2),\qquad \bm{\xi}\in[-1,1]^2,
	\label{eq:bubble}
\end{equation}
mapped to $\Omega_e$. The test functions are fixed once the partition is set; they carry no trainable parameters. Their locality makes each residual $R_k$ in \Cref{eq:weak_residuals} sensitive only to the patch $\mathrm{supp}(v_k)$, which localizes error control and improves conditioning relative to a single global residual. Construction details and the choice of $p$ appear in \Cref{subsec:quadrature}.

\paragraph{Weak residual and quadrature}
For each $v_k$ the weak residual \Cref{eq:weak_residuals} (forward) or \Cref{eq:inverse_residual} (inverse) is evaluated by mapping every background cell to the reference element and applying an $n_g$-point Gauss--Legendre rule,
\begin{equation}
	\int_{\Omega_e} g(\bm{x})\,d\Omega \approx \sum_{q=1}^{n_g^2} w_q\, g(\bm{x}(\bm{\xi}_q))\,|J_e(\bm{\xi}_q)|,
	\label{eq:gl_quad}
\end{equation}
with $J_e$ the cell Jacobian. Gauss--Legendre integration is exact for polynomial integrands up to degree $2n_g-1$ and, because the KAN trial field is piecewise-polynomial, delivers markedly lower integration error than the Monte-Carlo sampling used in energy-form baselines (\Cref{subsec:quadrature}). Gradients of $u_h$ inside \Cref{eq:weak_residuals} are obtained by automatic differentiation \cite{baydin2018automatic}.

\paragraph{Boundary conditions}
Neumann/traction data enter naturally through the boundary integral in \Cref{eq:weakform}. Dirichlet data are imposed either by a light penalty term $\mathcal{L}_{\Gamma^u}=\frac{1}{N_u}\sum_i |u_h(\bm{x}_i)-\bar u(\bm{x}_i)|^2$ or, when a distance function is convenient, by the admissible construction $u_h=u_p+D(\bm{x})\,u_g$ with $D$ vanishing on $\Gamma^u$. Because the weak form already encodes the interior physics and natural conditions, the number of tuned weights is far smaller than in \Cref{eq:strong_loss}.

\paragraph{Optimization}
The parameters are trained with Adam \cite{ADAM} followed by L-BFGS for the forward mechanics problems. The total loss is \Cref{eq:pg_loss} for forward problems and $\mathcal{L}_{\mathrm{PG}}^{\mathrm{inv}}$ for inverse problems.

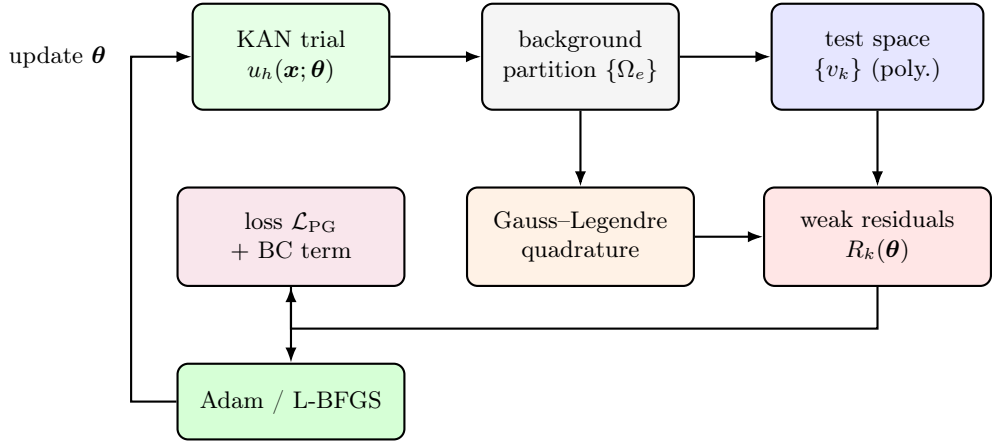
\begin{figure}[t]
\centering
\begin{tikzpicture}[font=\small,node distance=6mm]
  \node[block,fill=green!10,minimum height=1.4cm,minimum width=2.6cm] (kan) {KAN trial\\ $u_h(\bm{x};\bm{\theta})$};
  \node[block,fill=gray!8,right=12mm of kan,minimum height=1.4cm,minimum width=2.6cm] (grid) {background\\ partition $\{\Omega_e\}$};
  \node[block,fill=blue!10,right=12mm of grid,minimum height=1.4cm,minimum width=2.8cm] (test) {test space\\ $\{v_k\}$ (poly.)};
  \node[block,fill=orange!10,below=10mm of grid,minimum height=1.3cm,minimum width=3.0cm] (quad) {Gauss--Legendre\\ quadrature};
  \node[block,fill=red!10,below=10mm of test,minimum height=1.3cm,minimum width=3.0cm] (res) {weak residuals\\ $R_k(\bm{\theta})$};
  \node[block,fill=purple!10,below=10mm of kan,minimum height=1.3cm,minimum width=3.0cm] (loss) {loss $\mathcal{L}_{\mathrm{PG}}$\\ $+$ BC term};
  \node[block,fill=green!16,below=10mm of loss,minimum height=1.0cm,minimum width=3.0cm] (opt) {Adam / L-BFGS};

  \draw[edgearrow] (kan) -- (grid);
  \draw[edgearrow] (grid) -- (test);
  \draw[edgearrow] (grid) -- (quad);
  \draw[edgearrow] (test) -- (res);
  \draw[edgearrow] (quad) -- (res);
  \draw[edgearrow] (res.south) -- ($(res.south |- quad.south)+(0,-0.55)$)
    -- ($(loss.south |- quad.south)+(0,-0.55)$) -- (loss.south);
  \draw[edgearrow] (loss) -- (opt);
  \draw[edgearrow] (opt) -| ($(kan.west)+(-0.8,0)$) |- (kan.west) node[midway,left,align=center,xshift=-2mm]{update $\bm{\theta}$};
\end{tikzpicture}
\caption{The PG-KINN pipeline. A KAN provides the trial field $u_h$; an independent, compactly supported piecewise-polynomial test space $\{v_k\}$ is defined on a background partition; each pairing is integrated by Gauss--Legendre quadrature to form a weak residual $R_k$; the aggregated residual loss (plus a light boundary term) is minimized by Adam/L-BFGS, updating the KAN parameters $\bm{\theta}$.}
\label{fig:framework}
\end{figure}

\begin{algorithm}[t]
\caption{PG-KINN (Forward)}\label{alg:PG-KINN_forward}
\begin{algorithmic}[1]
\State \textbf{Input:} operator $\mathcal{P}$, data $f,\bar u,\bar t$; KAN architecture, grid $G$, order $r$; background partition $\{\Omega_e\}$; test degree $p$; GL points $n_g$
\State \textbf{(a) Build test basis:} construct $\{v_k\}_{k=1}^{M}$ from nodal hat and interior bubble functions (\Cref{eq:bubble}) on $\{\Omega_e\}$
\State \textbf{(b) Precompute quadrature:} map each $\Omega_e$ to $[-1,1]^d$; store GL nodes $\{\bm{\xi}_q,w_q\}$ and Jacobians $|J_e|$ (\Cref{eq:gl_quad})
\State Initialize KAN parameters $\bm{\theta}$; normalize inputs by \Cref{eq:normalize}
\For{epoch $=1,\dots,N_{\mathrm{ep}}$}
	\For{$k=1,\dots,M$}
		\State Evaluate $u_h$ and $\nabla u_h$ at GL nodes of $\mathrm{supp}(v_k)$ via AD
		\State \textbf{(c) Assemble residual:} $R_k(\bm{\theta})\gets a(u_h,v_k)-\ell(v_k)$ by \Cref{eq:weak_residuals} and \Cref{eq:gl_quad}
	\EndFor
	\State $\mathcal{L}\gets \frac{1}{M}\sum_k \omega_k R_k^2 + \lambda_u \mathcal{L}_{\Gamma^u}$
	\State \textbf{(d) Update:} $\bm{\theta}\gets$ Adam / L-BFGS step on $\nabla_{\bm\theta}\mathcal{L}$
\EndFor
\State \textbf{Output:} trained field $u_h(\bm{x};\bm{\theta})$
\end{algorithmic}
\end{algorithm}

\begin{algorithm}[t]
\caption{PG-KINN (Inverse)}\label{alg:PG-KINN_inverse}
\begin{algorithmic}[1]
\State \textbf{Input:} measured field $u$ (or $\bm{u}$), operator data $f,\bar t$; unknown coefficient $\kappa$ (or material parameters $\bm{\mu}$); KAN architecture for $\kappa_h$; partition $\{\Omega_e\}$; test basis $\{v_k\}$; GL points $n_g$
\State Build $\{v_k\}$ and precompute GL nodes/weights as in \Cref{alg:PG-KINN_forward} (steps a--b)
\State Initialize KAN parameters $\bm{\theta}_\kappa$ for $\kappa_h(\bm{x};\bm{\theta}_\kappa)$
\For{epoch $=1,\dots,N_{\mathrm{ep}}$}
	\For{$k=1,\dots,M$}
		\State Insert measured $u$ into \Cref{eq:inverse_residual}; assemble $R_k(\bm{\theta}_\kappa)$ at GL nodes
	\EndFor
	\State $\mathcal{L}^{\mathrm{inv}}\gets \frac{1}{M}\sum_k R_k(\bm{\theta}_\kappa)^2$ \Comment{no BC penalty; collapse-free by \Cref{prop:collapse}}
	\State Update $\bm{\theta}_\kappa$ with Adam / L-BFGS on $\nabla_{\bm{\theta}_\kappa}\mathcal{L}^{\mathrm{inv}}$
\EndFor
\State \textbf{Output:} identified coefficient $\kappa_h(\bm{x};\bm{\theta}_\kappa)$
\end{algorithmic}
\end{algorithm}

\Cref{alg:PG-KINN_forward,alg:PG-KINN_inverse} formalize the forward and inverse pipelines. For inverse problems the trainable KAN represents the unknown coefficient, the measured field is fixed inside each residual, and the boundary penalty is omitted from the trainable set. By \Cref{subsec:collapse} this loss cannot collapse to the trivial coefficient, unlike the energy form.

\subsection{Mathematical derivation of the Petrov--Galerkin weak form}\label{subsec:weakform}
We derive the weak form used in \Cref{subsec:pg} for a general second-order operator and then specialize to the model problems. Consider
\begin{equation}
	\mathcal{P}(u)=-\nabla\!\cdot\big(\bm\kappa\,\nabla u\big)+\bm b\cdot\nabla u + c\,u = f \ \text{in }\Omega,\qquad
	u=\bar u \ \text{on }\Gamma^u,\qquad (\bm\kappa\nabla u)\cdot\bm n=\bar t\ \text{on }\Gamma^t,
	\label{eq:app_general}
\end{equation}
where $\bm\kappa$ is a (possibly spatially varying and anisotropic) coefficient, $\bm b$ an advection field, and $c$ a reaction term; the operator is not required to be self-adjoint (the advection term $\bm b\cdot\nabla u$ breaks self-adjointness). Multiplying by a test function $v$ that vanishes on $\Gamma^u$ and integrating over $\Omega$,
\begin{equation}
	\int_\Omega\Big[-\nabla\!\cdot(\bm\kappa\nabla u)+\bm b\cdot\nabla u+c\,u-f\Big]v\,d\Omega=0.
	\label{eq:app_weighted}
\end{equation}
Applying the divergence theorem to the leading term,
\begin{equation}
	-\int_\Omega \nabla\!\cdot(\bm\kappa\nabla u)\,v\,d\Omega
	=\int_\Omega (\bm\kappa\nabla u)\cdot\nabla v\,d\Omega
	-\int_{\Gamma}\big[(\bm\kappa\nabla u)\cdot\bm n\big]\,v\,d\Gamma .
	\label{eq:app_ibp}
\end{equation}
Since $v=0$ on $\Gamma^u$, the boundary integral reduces to $\Gamma^t$, on which $(\bm\kappa\nabla u)\cdot\bm n=\bar t$. Substituting \Cref{eq:app_ibp} into \Cref{eq:app_weighted} gives the Petrov--Galerkin weak form: find $u$ with $u=\bar u$ on $\Gamma^u$ such that
\begin{equation}
	\underbrace{\int_\Omega (\bm\kappa\nabla u)\cdot\nabla v\,d\Omega
	+\int_\Omega (\bm b\cdot\nabla u)\,v\,d\Omega
	+\int_\Omega c\,u\,v\,d\Omega}_{a(u,v)}
	=\underbrace{\int_\Omega f\,v\,d\Omega+\int_{\Gamma^t}\bar t\,v\,d\Gamma}_{\ell(v)},\qquad\forall v\in\mathcal V.
	\label{eq:app_bilinear}
\end{equation}
The bilinear form $a(\cdot,\cdot)$ need not be symmetric, which is precisely why a Petrov--Galerkin (independent test space) statement is appropriate: the Bubnov--Galerkin choice $v=\delta u$ would presuppose symmetry (self-adjointness). For the Poisson/heat model in \Cref{subsec:pg} we set $\bm\kappa=\kappa\bm I$, $\bm b=\bm 0$, $c=0$, recovering \Cref{eq:weakform}. Replacing $v$ by the finite test basis $\{v_k\}$ yields the residuals $R_k=a(u_h,v_k)-\ell(v_k)$ of \Cref{eq:weak_residuals}. For the elasticity problems, the same steps applied to the momentum balance $-\nabla\!\cdot\bm\sigma=\bm f$ give the virtual-work residual $\int_\Omega \bm\sigma:\nabla v_k\,d\Omega-\int_{\Gamma^t}\bar{\bm t}\cdot v_k\,d\Gamma-\int_\Omega\bm f\cdot v_k\,d\Omega$, and for finite deformation $\bm\sigma$ is replaced by the first Piola--Kirchhoff stress $\bm P=\partial\Psi/\partial\bm F$ with the integral taken over the reference configuration \cite{bonet1997nonlinear}. Only first derivatives of $u_h$ appear, one order lower than the strong form \Cref{eq:strong_loss}.

\subsection{Collapse-freeness of the Petrov--Galerkin inverse loss}\label{subsec:collapse}
We show that the residual-based inverse loss of \Cref{subsec:inverse_pg} does not admit the trivial-coefficient minimizer that afflicts the energy-based inverse problem.

\paragraph{Energy-based inverse problem collapses}
Consider linear elasticity with unknown modulus $E$ (Poisson ratio $\nu$ fixed) and \emph{prescribed} displacement data $\bm u$, tractions $\bar{\bm t}$, and body force $\bm f$. The deep energy method minimizes
\begin{equation}
	\mathcal{L}_{\mathrm{DEM}}=\int_\Omega \Psi\,d\Omega-\int_{\Gamma^t}\bar{\bm t}\cdot\bm u\,d\Gamma-\int_\Omega \bm f\cdot\bm u\,d\Omega,\qquad
	\Psi=\tfrac12\sigma_{ij}\varepsilon_{ij},\ \ \sigma_{ij}=2G\varepsilon_{ij}+\lambda\varepsilon_{kk}\delta_{ij}.
	\label{eq:app_dem}
\end{equation}
Because $\bm u$, $\bar{\bm t}$, $\bm f$ are data, the last two integrals are constants in the optimization over $E$, and $\varepsilon_{ij}$ is fixed. With $G=\tfrac{E}{2(1+\nu)}$ and $\lambda=\tfrac{E\nu}{(1+\nu)(1-2\nu)}$ both proportional to $E$, the strain energy is $\Psi=E\,\Psi_0$ with $\Psi_0\ge0$ independent of $E$. Hence
\begin{equation}
	\arg\min_{E\ge0}\ \int_\Omega \Psi\,d\Omega=\arg\min_{E\ge0}\ E\!\int_\Omega\Psi_0\,d\Omega = 0,
	\label{eq:app_collapse}
\end{equation}
so the energy form drives $E\to0$: a spurious global minimizer with zero strain energy. This is the zero-modulus collapse.

\paragraph{Petrov--Galerkin residual does not collapse}
Consider the same class of inverse problems in weak-residual form. For the conductivity problem (the elasticity case is identical with $\kappa\!\to\! E$ and gradients replaced by symmetric gradients), the residuals are affine in the unknown field,
\begin{equation}
	R_k(\bm\theta_\kappa)=\underbrace{\int_\Omega \kappa_h(\bm x;\bm\theta_\kappa)\,\nabla u\cdot\nabla v_k\,d\Omega}_{\text{linear in }\kappa_h}-\underbrace{\Big(\int_{\Gamma^t}\bar t\,v_k\,d\Gamma+\int_\Omega f\,v_k\,d\Omega\Big)}_{=:\,\ell_k\ \text{(data, constant)}} .
	\label{eq:app_res_affine}
\end{equation}
Setting the trivial coefficient $\kappa_h\equiv 0$ gives $R_k=-\ell_k$, and therefore
\begin{equation}
	\mathcal{L}^{\mathrm{inv}}_{\mathrm{PG}}(\kappa_h\equiv0)=\frac{1}{M}\sum_{k=1}^M \ell_k^2 .
	\label{eq:app_nonzero}
\end{equation}
Whenever the data are nontrivial---i.e.\ there exists at least one test function with $\ell_k\neq0$, which holds for any problem with a nonzero source or nonzero Neumann flux---we have $\mathcal{L}^{\mathrm{inv}}_{\mathrm{PG}}(\kappa_h\equiv0)>0$, so $\kappa_h\equiv0$ is \emph{not} a minimizer.
\begin{prop}[Collapse-freeness]\label{prop:collapse}
Let $u$ be exact data generated by the true coefficient $\kappa^\star$, so that $R_k(\kappa^\star)=0$ for all $k$. Then $\mathcal{L}^{\mathrm{inv}}_{\mathrm{PG}}(\kappa^\star)=0$, while $\mathcal{L}^{\mathrm{inv}}_{\mathrm{PG}}(0)=\frac1M\sum_k\ell_k^2>0$ whenever the data are nontrivial. Consequently the trivial coefficient is never a global minimizer, and if the test space is rich enough that $\{v_k\}$ determines $\kappa$ from $\int_\Omega\kappa\,\nabla u\cdot\nabla v_k$ uniquely, the true coefficient $\kappa^\star$ is the unique global minimizer.
\end{prop}
\begin{proof}
$R_k(\kappa^\star)=0$ gives $\mathcal{L}^{\mathrm{inv}}_{\mathrm{PG}}(\kappa^\star)=0$, the global lower bound of a sum of squares. \Cref{eq:app_nonzero} gives $\mathcal{L}^{\mathrm{inv}}_{\mathrm{PG}}(0)>0$. Uniqueness follows because $R_k$ is affine in $\kappa_h$: the minimizers form the affine solution set of $\int_\Omega\kappa_h\nabla u\cdot\nabla v_k=\ell_k$ for all $k$, which is the singleton $\{\kappa^\star\}$ under the stated identifiability of the test space.
\end{proof}
The contrast with \Cref{eq:app_collapse} is the crux: the energy form discards the external-work data during optimization (they become additive constants), whereas the weak residual keeps them inside each $R_k$ through $\ell_k$, anchoring the minimizer to the true coefficient. This is why PG-KINN is used only in weak-residual form for inverse problems.

\subsection{Similarities between KAN B-splines and finite element shape functions}\label{subsec:kanfem}
The affinity between PG-KINN and classical discretizations rests on the fact that a KAN edge is a spline expansion, and splines are the bases of FEM and isogeometric analysis. Consider a 1D linear finite element mesh of $N$ elements,
\begin{equation}
	u^{\mathrm{fem}}(x)=\sum_{i=1}^{N+1}N_i(x)\,u_i,
	\label{eq:app_fem1d}
\end{equation}
with hat shape functions $N_i$, and a $[1,1]$ KAN of grid size $N$,
\begin{equation}
	u^{\mathrm{kan}}(x)=\sum_{i=1}^{N+k}B_i(x)\,c_i,
	\label{eq:app_kan1d}
\end{equation}
neglecting the residual/scaler terms for clarity. For first-order splines ($k=1$) one has $B_i(x)=N_i(x)$ exactly, so a single KAN edge \emph{is} a linear finite element expansion. For quadratic order ($k=2$, grid size $2$) the four B-spline coefficients $\{c_1,\dots,c_4\}$ are determined by the three nodal values $\{u_1,u_2,u_3\}$: matching $u^{\mathrm{fem}}=u^{\mathrm{kan}}$ on $[-1,0]$ and $[0,1]$ yields
\begin{equation}
	\begin{bmatrix}c_1\\c_2\\c_3\\c_4\end{bmatrix}
	=\begin{bmatrix}1.75&-1&0.25\\0.25&1&-0.25\\-0.25&1&0.25\\0.25&-1&1.75\end{bmatrix}
	\begin{bmatrix}u_1\\u_2\\u_3\end{bmatrix},
	\label{eq:app_quad_map}
\end{equation}
so quadratic FEM is contained in the quadratic KAN. In two dimensions the correspondence is exact once the additive KAN combination is made multiplicative: a $[2,1]$ KAN with grid size $N$ and order $1$ gives $K(\bm x)=\sum_i N_i^x(x)c_i^x+\sum_i N_i^y(y)c_i^y$, whereas the tensor-product FEM field is $u(\bm x)=\sum_i N_i(\bm x)u_i$ with $N_i=N_i^x N_i^y$; replacing the sum of the two edge maps by their product,
\begin{equation}
	K_{\mathrm{mult}}(\bm x)=\Big(\sum_i N_i^x(x)c_i^x\Big)\Big(\sum_j N_j^y(y)c_j^y\Big)=\sum_{i,j}N_i^x(x)N_j^y(y)\,c_i^x c_j^y,
	\label{eq:app_2d}
\end{equation}
reproduces the FEM tensor-product basis. Since a deep KAN composes such spline maps, $u^{\mathrm{kan}}=S(S(\cdots S(\bm x)))$ is a nested shape-function expansion, and because B-splines are a special case of NURBS, the same reasoning links KAN to isogeometric analysis. This structural kinship is what makes the KAN trial space a natural partner for the weak (Petrov--Galerkin) forms used throughout the paper: both are piecewise-polynomial, so quadrature is accurate and interfaces are captured by the $C^0$ structure.

\subsection{Numerical integration and construction of piecewise-polynomial test functions}\label{subsec:quadrature}
\paragraph{Gauss--Legendre quadrature}
Every weak residual \Cref{eq:weak_residuals} is a sum of cell integrals. On a background cell $\Omega_e$ we map to the reference square $[-1,1]^2$ and apply a tensor-product Gauss--Legendre rule,
\begin{equation}
	\int_{\Omega_e} g\,d\Omega=\int_{[-1,1]^2} g(\bm x(\bm\xi))\,|J_e(\bm\xi)|\,d\bm\xi
	\approx\sum_{p=1}^{n_g}\sum_{q=1}^{n_g} w_p w_q\, g\big(\bm x(\xi_p,\xi_q)\big)\,|J_e(\xi_p,\xi_q)|,
	\label{eq:app_gl}
\end{equation}
where $\{\xi_p,w_p\}$ are the 1D Gauss--Legendre nodes and weights (roots of the Legendre polynomial $L_{n_g}$), computed once by the Golub--Welsch algorithm \cite{golub2013matrix,quarteroni2008numerical}. The rule is exact for polynomial integrands of degree $\le 2n_g-1$. Because both the KAN trial field and the polynomial test functions are piecewise-polynomial, the integrand $\bm\kappa\nabla u_h\cdot\nabla v_k$ is nearly polynomial on each cell, and a modest $n_g$ (we use $n_g=4$--$6$) integrates it to near machine precision for the tested settings. This is the mechanism behind the beam experiment (\Cref{subsubsec:beam}), where Gauss--Legendre quadrature reduces the error by more than an order of magnitude relative to Monte-Carlo sampling.

\paragraph{Piecewise-polynomial test functions}
The test space is spanned by two families defined on the partition $\{\Omega_e\}$.
\emph{(i) Nodal $C^0$ functions.} Tensor products of 1D hat (degree 1) or Lagrange (degree $p$) functions, continuous across cell faces and supported on the cells adjacent to a node.
\emph{(ii) Interior bubble functions.} On each cell, shifted Legendre polynomials multiplied by a vanishing envelope,
\begin{equation}
	v^{\mathrm{bub}}_{e,p}(\bm\xi)=(1-\xi_1^2)(1-\xi_2^2)\,L_p(\xi_1)\,L_p(\xi_2),\qquad p=0,1,\dots,p_{\max},
	\label{eq:app_bubble}
\end{equation}
which vanish on $\partial\Omega_e$ and are therefore admissible ($v=0$ on $\Gamma^u$ is enforced by excluding boundary-node functions on $\Gamma^u$). The Legendre factor makes the family hierarchical and nearly $L_2$-orthogonal, so distinct $v_k$ probe distinct spectral content of the residual, which flattens the effective conditioning of the loss \Cref{eq:pg_loss}. The number of test functions per cell, $M_e=(p_{\max}+1)^2$, controls how finely the weak form is sampled; we use $p_{\max}=2$--$4$ depending on the field complexity, matched to the KAN grid size. The partition for the test space and the spline grid of the KAN are chosen independently: refining the test partition sharpens the physics measurement, while refining the KAN grid sharpens the trial representation.

\section{Numerical Experiments}\label{sec:experiments}
We evaluate PG-KINN on benchmarks drawn from computational solid mechanics. Unless stated otherwise, all methods share the same optimizer, learning rate ($10^{-3}$ for Adam), and sampling budget. We report the relative $\mathcal{L}_2$ error
\begin{equation}
	\phi_{\mathcal{L}_2}=\frac{\int_\Omega |\phi_{\mathrm{pred}}-\phi_{\mathrm{exact}}|^2\,d\Omega}{\int_\Omega |\phi_{\mathrm{exact}}|^2\,d\Omega},
	\label{eq:relL2}
\end{equation}
and, where stress accuracy matters, the relative $\mathcal{H}_1$ error of the Mises field. Reference solutions are analytical where available and FEM otherwise. Two baseline families are compared:
\begin{itemize}
	\item \textbf{Baseline 1 (Legacy MLP):} strong-form PINN-MLP, energy-form DEM-MLP, inverse-form BINN-MLP, and Petrov--Galerkin PG-MLP (VPINN-style \cite{hp-VPINN}).
	\item \textbf{Baseline 2 (SOTA KINN):} PIKAN\_CPINNs (strong+KAN), PIKAN\_DEM (energy+KAN), and PIKAN\_BINN (inverse+KAN) \cite{liu2024kan}, which isolate the effect of the formulation when the backbone is held fixed.
\end{itemize}
PG-KINN denotes the Petrov--Galerkin KAN of \Cref{sec:framework}; its errors are highlighted in \textbf{bold} in the tables. Computations use a single GPU.

\subsection{Singularity and Discontinuity Problems}\label{subsec:singular}

\subsubsection{Mode III Crack (Singularity)}\label{subsubsec:crack}
The anti-plane (Mode III) crack is a canonical singular problem on $\Omega=[-1,1]^2$. The strong-form governing equations and boundary conditions are
\begin{equation}
\begin{cases}
	\Delta u=0, & \bm{x}\in\Omega,\\
	\bar u(\bm{x})=r^{1/2}\sin(\theta/2), & \bm{x}\in\Gamma,
\end{cases}
\label{eq:crack_pde}
\end{equation}
where $(r,\theta)$ are polar coordinates and $\theta\in[-\pi,\pi]$, so that $u$ is discontinuous across the crack face $x<0,\,y=0$. The crack-tip strain $\varepsilon_{z\theta}=\tfrac{1}{2\sqrt{r}}$ is singular at the origin. Baseline strong-form CPINNs \cite{CPINN} minimize
\begin{equation}
\begin{split}
	\mathcal{L}_{\mathrm{CPINNs}} &= \lambda_{1}\!\sum_{i=1}^{N_{+}}\!|\Delta u^{+}(\bm{x}_i)|^{2}
	+\lambda_{2}\!\sum_{i=1}^{N_{-}}\!|\Delta u^{-}(\bm{x}_i)|^{2}
	+\lambda_{3}\!\sum_{i=1}^{N_{b+}}\!|u^{+}-\bar u|^{2}
	+\lambda_{4}\!\sum_{i=1}^{N_{b-}}\!|u^{-}-\bar u|^{2}\\
	&+\lambda_{5}\!\sum_{i=1}^{N_{I}}\!|u^{-}-u^{+}|^{2}
	+\lambda_{6}\!\sum_{i=1}^{N_{I}}\!|\bm{n}\!\cdot\!(\nabla u^{+}-\nabla u^{-})|^{2},
\end{split}
\label{eq:cpinns_loss}
\end{equation}
with $\{\lambda_i\}_{i=1}^{6}=\{1,1,50,50,10,10\}$; the energy-form DEM baseline minimizes $\mathcal{L}_{\mathrm{DEM}}=\int_\Omega \tfrac12|\nabla u|^2\,d\Omega$ with hard essential BCs. PG-KINN has no such hyperparameters: it minimizes $\mathcal{L}_{\mathrm{PG}}$ from \Cref{eq:pg_loss}.

We compare the strong-form PINN, the energy-form DEM, PG-MLP, and PG-KINN. \Cref{tab:crack} quantifies the comparison: PG-KINN attains the lowest relative $\mathcal{L}_2$ error.

\begin{table}[t]
\caption{Mode III crack: accuracy. Baseline~1: legacy MLP; Baseline~2: PIKAN (KINN).}
\label{tab:crack}
\centering
\begin{adjustbox}{max width=\textwidth}
\begin{tabular}{lccccc}
\toprule
Method & Rel. $\mathcal{L}_2$ error & Grid size & Grid range & Order & Architecture\\
\midrule
\multicolumn{6}{l}{\textit{Baseline 1 (Legacy MLP)}}\\
PINN-MLP (strong)      & 0.06254 & -- & -- & -- & [2,30,30,30,30,1]\\
DEM-MLP (energy)       & 0.02429 & -- & -- & -- & [2,30,30,30,30,1]\\
PG-MLP (VPINN)         & 0.01810 & -- & -- & -- & [2,30,30,30,30,1]\\
\midrule
\multicolumn{6}{l}{\textit{Baseline 2 (SOTA KINN)}}\\
PIKAN\_CPINNs (strong+KAN) & 0.02988 & 10 & [-1,1] & 3 & [2,5,5,1]\\
PIKAN\_DEM (energy+KAN)    & 0.01117 & 10 & [-1,1] & 3 & [2,5,5,1]\\
PIKAN\_BINN (inverse+KAN)  & 0.00083 & 10 & [-1,1] & 3 & [2,5,5,5,1]\\
\midrule
\textbf{PG-KINN (PG+KAN)} & \textbf{0.00081} & 10 & [-1,1] & 3 & [2,5,5,1]\\
\bottomrule
\end{tabular}
\end{adjustbox}
\end{table}

\subsubsection{Plate with a Central Hole (Stress Concentration)}\label{subsubsec:plate}
The plate with a central hole tests stress concentration in linear elasticity. On the quarter model $\Omega$ ($L=20$~mm) the governing equations are
\begin{equation}
\begin{cases}
	\sigma_{ij,j}=0, & \bm{x}\in\Omega,\\
	\sigma_{ij}=\dfrac{E}{1+\nu}\varepsilon_{ij}+\dfrac{E\nu}{(1+\nu)(1-2\nu)}\varepsilon_{kk}\delta_{ij}, & \bm{x}\in\Omega,\\
	\varepsilon_{ij}=\tfrac12(u_{i,j}+u_{j,i}), & \bm{x}\in\Omega,\\
	u_x=0 \text{ on } x=0,\quad u_y=0 \text{ on } y=0,\quad \sigma_{ij}n_j=\bar t_i \text{ on } \Gamma^t,
\end{cases}
\label{eq:plate_pde}
\end{equation}
with $E=1000$~MPa, $\nu=0.3$ (plane stress), hole radius $5$~mm, and applied traction $t_x=100$~N/mm on the right edge. The strong-form PINN baseline minimizes a heavily weighted residual with eight penalty terms \cite{liu2024kan}; the DEM baseline minimizes $\mathcal{L}_{\mathrm{DEM}}=\int_\Omega \Psi\,d\Omega-\int_{\Gamma^{\mathrm{right}}}t_x u_x\,d\Gamma$ with $\Psi=\tfrac12\sigma_{ij}\varepsilon_{ij}$ and admissible displacements $u_x=x\hat u_x$, $u_y=y\hat u_y$.

\Cref{tab:plate} reports the final displacement ($\mathcal{L}_2$) and Mises-stress ($\mathcal{H}_1$) errors. PG-KINN attains the lowest displacement error and the best Mises-stress accuracy among the Petrov--Galerkin and strong-form baselines.

\begin{table}[t]
\caption{Plate with a central hole: accuracy.}
\label{tab:plate}
\centering
\begin{adjustbox}{max width=\textwidth}
\begin{tabular}{lcccccc}
\toprule
Method & Disp. $\mathcal{L}_2$ & Mises $\mathcal{H}_1$ & Grid size & Grid range & Order & Architecture\\
\midrule
\multicolumn{7}{l}{\textit{Baseline 1 (Legacy MLP)}}\\
PINN-MLP (strong) & 0.00858 & 0.00649 & -- & -- & -- & [2,30,30,30,30,2]\\
DEM-MLP (energy)  & 0.00154 & 0.01150 & -- & -- & -- & [2,30,30,30,30,2]\\
PG-MLP (VPINN)    & 0.00110 & 0.00350 & -- & -- & -- & [2,30,30,30,30,2]\\
\midrule
\multicolumn{7}{l}{\textit{Baseline 2 (SOTA KINN)}}\\
PIKAN\_CPINNs (strong+KAN) & 0.00629 & 0.00481 & 15 & [0,1] & 3 & [2,5,5,5,2]\\
PIKAN\_DEM (energy+KAN)    & 0.00092 & 0.00517 & 15 & [0,1] & 3 & [2,5,5,5,2]\\
PIKAN\_BINN (inverse+KAN)  & 0.00023 & 0.00211 & 10 & [0,20] & 3 & [2,5,5,5,2]\\
\midrule
\textbf{PG-KINN (PG+KAN)} & \textbf{0.00035} & \textbf{0.00160} & 15 & [0,1] & 3 & [2,5,5,5,2]\\
\bottomrule
\end{tabular}
\end{adjustbox}
\end{table}

\subsection{Nonlinear Hyperelasticity (Neo-Hookean Cantilever Beam)}\label{subsubsec:beam}
We test geometrically and materially nonlinear solid mechanics. On the reference domain $\Omega=[0,L]\times[0,H]$ ($L=4$, $H=1$) the governing equations are
\begin{equation}
\begin{cases}
	\nabla\!\cdot\!\bm{P}=\bm{0}, & \bm{X}\in\Omega,\\
	\bm{P}=\partial\Psi/\partial\bm{F},\quad \bm{F}=\partial\bm{x}/\partial\bm{X},\quad \bm{x}=\bm{X}+\bm{u}, & \bm{X}\in\Omega,\\
	\bm{u}=\bm{0} \text{ at } X_1=0,\quad \bar{\bm{t}}=(0,-5)^{\mathsf T} \text{ at } X_1=L,
\end{cases}
\label{eq:beam_pde}
\end{equation}
with Neo-Hookean strain energy
\begin{equation}
	\Psi=\tfrac{1}{2}\lambda(\ln J)^2-\mu\ln J+\tfrac{1}{2}\mu(I_1-3),\quad J=\det\bm F,\ I_1=\mathrm{tr}(\bm F^\top\bm F),
	\label{eq:neohookean}
\end{equation}
Lam\'e parameters $\lambda,\mu$ from $E=1000$, $\nu=0.3$. The DEM baseline minimizes $\mathcal{L}_{\mathrm{DEM}}=\int_\Omega(\Psi-\bm{f}\!\cdot\!\bm{u})\,d\Omega-\int_{\Gamma^t}\bar{\bm{t}}\!\cdot\!\bm{u}\,d\Gamma$ \cite{liu2024kan}. PG-KINN forms the Petrov--Galerkin residual from the first variation of the total potential, i.e.\ $\int_\Omega \bm P:\nabla v_k\,d\Omega-\int_{\Gamma^t}\bar{\bm t}\cdot v_k\,d\Gamma=0$.

Because the energy integral must be evaluated accurately, we compare quadrature rules. \Cref{tab:beam} reports that PG-KINN with Gauss--Legendre quadrature is the most accurate configuration; low-order rules (trapezoidal) already suffice for the KAN trial space, while Monte-Carlo integration should be avoided.

\begin{table}[t]
\caption{Neo-Hookean cantilever beam: accuracy.}
\label{tab:beam}
\centering
\begin{adjustbox}{max width=\textwidth}
\begin{tabular}{lccccc}
\toprule
Method & Disp. $\mathcal{L}_2$ & Grid size & Grid range & Order & Architecture\\
\midrule
\multicolumn{6}{l}{\textit{Baseline 1 (Legacy MLP)}}\\
DEM-MLP, Monte-Carlo      & 0.04756 & -- & -- & -- & [2,30,30,30,30,2]\\
DEM-MLP, Simpson          & 0.01753 & -- & -- & -- & [2,30,30,30,30,2]\\
\midrule
\multicolumn{6}{l}{\textit{Baseline 2 (SOTA KINN)}}\\
PIKAN\_DEM, Monte-Carlo   & 0.03334 & 15 & [0,1] & 3 & [2,5,5,5,2]\\
PIKAN\_DEM, trapezoidal   & 0.00254 & 15 & [0,1] & 3 & [2,5,5,5,2]\\
PIKAN\_DEM, Simpson       & 0.00218 & 15 & [0,1] & 3 & [2,5,5,5,2]\\
\midrule
\textbf{PG-KINN, Gauss--Legendre} & \textbf{0.00192} & 15 & [0,1] & 3 & [2,5,5,5,2]\\
\bottomrule
\end{tabular}
\end{adjustbox}
\end{table}

\subsection{Inverse Problems: Parameter Identification in Heterogeneous Media}\label{subsec:inverse}
We now identify an unknown material field from measured responses using the collapse-free inverse formulation of \Cref{subsec:inverse_pg}.

\subsubsection{Highly Heterogeneous Thermal Conductivity (Famous Paintings)}\label{subsubsec:paintings}
To stress-test the fitting capacity, we identify a spatially varying conductivity $k(x,y)$ on $[0,1]^2$ from a manufactured temperature field. The forward problem is
\begin{equation}
\begin{cases}
	-\nabla\!\cdot(k\nabla T)=f, & (x,y)\in[0,1]^2,\\
	T(x,y)=\cos(15\pi xy), & f=\cos(15\pi xy)\big[(15\pi x)^2+(15\pi y)^2\big]k(x,y),
\end{cases}
\label{eq:paint_pde}
\end{equation}
with $k$ given by three grayscale textures (Picasso-like P, Caspar-like C, Van Gogh-like V) \cite{chen2021learning,liu2024multi}. PG-KINN minimizes $\mathcal{L}^{\mathrm{inv}}_{\mathrm{PG}}$ of \Cref{alg:PG-KINN_inverse} with $k_h$ represented by a KAN. \Cref{tab:paint} lists the data-driven and physics-based inverse errors: the KAN trial space substantially outperforms the MLP on high-contrast textures, and the collapse-free weak residual recovers the conductivity field directly from the PDE.

\begin{table}[t]
\caption{Highly heterogeneous conductivity identification: relative $\mathcal{L}_2$ errors. P: Picasso-like, C: Caspar-like, V: Van Gogh-like.}
\label{tab:paint}
\centering
\begin{adjustbox}{max width=\textwidth}
\begin{tabular}{lccccc}
\toprule
Method & Data-driven $\mathcal{L}_2$ & PDE inverse $\mathcal{L}_2$ & Grid size & Order & Architecture\\
\midrule
MLP        & P: 48.6\%; C: 21.5\%; V: 35.0\% & -- & -- & -- & [2,200,200,1]\\
\textbf{PG-KINN (KAN)} & P: 21.1\%; C: 3.4\%; V: 12.9\% & P: 31.2\%; C: 5.5\%; V: 20.2\% & 100 & 3 & [2,15,15,1]\\
\bottomrule
\end{tabular}
\end{adjustbox}
\end{table}

\subsection{Performance on Complex Geometries}\label{subsec:complex}
We probe the limitation of the spline trial space on \emph{extreme} complex geometries. We solve
\begin{equation}
\begin{cases}
	\Delta u=0, & \bm{x}\in\Omega,\\
	u=\sin(x)\sinh(y)+\cos(x)\cosh(y), & \bm{x}\in\partial\Omega,
\end{cases}
\label{eq:complex_pde}
\end{equation}
on a Koch snowflake and a flower-shaped domain. \Cref{tab:complex} reports that PG-KINN on the Koch snowflake attains $e=9.5\times10^{-3}$, beating PIKAN\_CPINNs ($0.01188$) and matching PIKAN\_BINN on the flower problem. The reason is structural: the KAN spline grid is axis-aligned and rectangular, so a highly non-rectangular boundary forces a mismatch between the simplicity of the target field and the grid needed to conform to the geometry. This is consistent with the known difficulty of spline bases on fractal boundaries and motivates adaptive or conformal grids for extreme cases.

\begin{table}[t]
\caption{Complex geometries: relative $\mathcal{L}_2$ error.}
\label{tab:complex}
\centering
\begin{adjustbox}{max width=\textwidth}
\begin{tabular}{lcccc}
\toprule
Method & Rel. $\mathcal{L}_2$ error & Grid size & Order & Architecture\\
\midrule
\multicolumn{5}{l}{\textit{Baseline 1 (Legacy MLP)}}\\
PINN-MLP (strong, Koch)  & 0.01939 & -- & -- & [2,30,30,30,30,1]\\
DEM-MLP (Koch)           & 0.01563 & -- & -- & [2,30,30,30,30,1]\\
\midrule
\multicolumn{5}{l}{\textit{Baseline 2 (SOTA KINN)}}\\
PIKAN\_CPINNs (Koch)     & 0.01188 & 10 & 2 & [2,5,5,5,1]\\
PIKAN\_DEM (Koch)        & 0.01565 & 3 & 2 & [2,5,5,1]\\
PIKAN\_BINN (flower)     & 0.00473 & 10 & 3 & [2,5,5,5,1]\\
\midrule
\textbf{PG-KINN (Koch)}    & \textbf{0.00950} & 10 & 2 & [2,5,5,5,1]\\
\textbf{PG-KINN (flower)}  & 0.00473 & 10 & 3 & [2,5,5,5,1]\\
\bottomrule
\end{tabular}
\end{adjustbox}
\end{table}

\section{Conclusion}\label{sec:conclusion}
We introduced PG-KINN, a physics-informed Kolmogorov--Arnold network built on a Petrov--Galerkin weak formulation. The framework pairs a KAN trial space, whose local spline structure mirrors the shape functions of finite elements, with an \emph{independent}, compactly supported, piecewise-polynomial test space evaluated by Gauss--Legendre quadrature. This combination inherits the low derivative order of the energy form, the operator generality of the strong form, and---unlike either---a residual-based loss that remains well posed for inverse problems. We proved that this inverse loss cannot collapse to the trivial coefficient that ruins energy-based identification, and we derived the weak form for general operators. Formal forward and inverse algorithms (\Cref{alg:PG-KINN_forward,alg:PG-KINN_inverse}) make the pipeline directly implementable.

Across a benchmark suite spanning crack singularities, stress concentration, Neo-Hookean hyperelasticity, highly heterogeneous inverse identification, and complex geometries---PG-KINN was consistently more accurate than legacy MLP baselines and state-of-the-art KAN-based strong/energy/inverse formulations (PIKAN). It identified conductivity fields robustly in the inverse setting. The remaining limitation is \emph{extreme} complex geometries (Koch snowflake, flower), where the axis-aligned spline grid is a bottleneck; here PG-KINN is competitive with PIKAN on the Koch domain but does not match PIKAN\_BINN on the flower problem.

Several directions follow. Adaptive isoparametric refinement and conformal mappings can improve accuracy on irregular boundaries; $h$-adaptive refinement of the KAN grid remains attractive. On the efficiency side, the per-epoch cost of B-spline evaluation can be reduced with faster bases such as Chebyshev polynomials \cite{ss2024chebyshev,shukla2024comprehensive}. On the formulation side, adaptive and adversarial choices of the test space \cite{zang2020weak} and rigorous quadrature/test-function analysis \cite{berrone2022variational} may further tighten accuracy. We believe the Petrov--Galerkin coupling of learnable spline trial spaces with fixed polynomial test spaces is a principled and extensible foundation for AI-based computational mechanics.

\section*{Declaration of competing interest}
The authors declare that they have no known competing financial interests or personal relationships that could have appeared to influence the work reported in this paper.

\section*{Acknowledgements}
The authors acknowledge that ChatGPT, available at https://openai.com/chatgpt, was used for language editing and proofreading purposes. All technical content, analyses, and conclusions were entirely developed and written by the authors.

\bibliographystyle{elsarticle-num}
\bibliography{PG-KINN}

\end{document}